\documentclass[12pt]{article}
\usepackage{latexsym}
\usepackage{amsmath,amsthm,amsopn,amssymb}
\usepackage{bbding}
\usepackage{graphicx}
\usepackage{hyperref}
\usepackage{mathptmx}
\allowdisplaybreaks

\numberwithin{equation}{section}
\newtheorem{theorem}{Theorem}[section]
\newtheorem{lemma}{Lemma}[section]

\newtheorem{algorithm}{Algorithm}[section]

\makeatother

\newcommand{\M}{{\mathbb{M}}}

\begin{document}

\title{On Regularization of Convolutional Kernels in Neural Networks
%\thanks{}
}

\author{
Pei-Chang Guo \thanks{Research supported in part by  China Scholarship Council. This work was done while this author was a visiting scholar at the Department of Mathematics, University of Kentucky, from July 2018 to July 2019.   E-mail:peichang@cugb.edu.cn} \\
School of Science,
China University of Geosciences, Beijing, 100083, China\\
Qiang Ye \thanks{Research supported in part
by NSF under grants DMS-1821144 and DMS-1620082. Email:qye3@uky.edu}\\
Department of Mathematics, University of Kentucky,\\
Lexington, KY, 40508, United States
}

\date{}

\maketitle
\begin{abstract}
Convolutional neural network is an important model in deep learning. To avoid  exploding/vanishing gradient problems and to improve the generalizability of a neural network, it is desirable to have a convolution operation that nearly preserves the norm, or to have the singular values of the transformation matrix corresponding to a convolutional kernel bounded around $1$.  We propose a penalty function that can be used in the optimization of a convolutional neural network to constrain the singular values of  the  transformation matrix around  $1$.  We derive an algorithm to carry out the gradient descent minimization of this penalty function in terms of convolution kernels. Numerical examples are presented to demonstrate the effectiveness of the method.
%The penalty is about the transformation matrix corresponding to a kernel, not directly about the kernel, which is different from results in existing papers. This provides a new regularization method about  the weights of convolutional layers.
%Other penalty functions about a kernel can be devised following this idea in future.

\vspace{2mm} \noindent \textbf{Keywords}:  penalty function, transformation matrix, convolutional layers, generalizability, unstable gradient.
\end{abstract}

\section{Introduction}

%Convolution without flip is an important operation in the field of deep learning  \cite{dumoulin2018}. Depending on different strides and padding patterns, there are many different forms of convolution  \cite{dumoulin2018}. Without loss of  generality,  we will consider in this paper the convolution with unit strides.
%Our objective is to find a convolution kernel $K$ so that its corresponding  transformation matrix $M$ has some desirable properties, such as having singular values bounded around $1$ or being well-conditioned.

% First we introduce the convolution arithmetic in deep learning,  which is different from

The classical convolution operation is an essential tool in signal processing. More general forms of convolution that use no flip in multiplications but may have different strides and padding patterns have been introduced and widely used in deep learning  \cite{dumoulin2018}. Here, only element-wise multiplication and addition are performed and there is no reverse multiplication with the convolutional kernel. Without loss of  generality,  we will consider in this paper 2-dimensional convolutions with unit strides and with zero padding.
% We use $*$ to denote the convolution arithmetic in deep learning and $\ulcorner \cdot\urcorner$ is to round a number to the nearest integer greater than or equal to that number.
Specifically, given a convolutional kernel matrix $K=[k_{ij}]\in \mathbb{R}^{k\times k}$  and an input matrix  $X=[x_{ij}]\in \mathbb{R}^{N\times N}$, we consider the convolution of $K$ and $X$, denoted by $Y = K*X \in \mathbb{R}^{N \times N}$,  as defined by
 %each entry of the output $Y\in \mathbb{R}^{N \times N}$  is produced by
\begin{equation}\label{eq:1channel}
    y_{r,s}= (K*X)_{r,s}= \sum_{p\in \{1,\cdots,k\} }\sum_{q\in \{1,\cdots,k\}} x_{r-m+p,s-m+q}k_{p,q},
\end{equation}
where $m=\ulcorner k/2\urcorner$,  and $x_{i,j}=0$ if $i\leq 0$ or $i> N$, or $j\leq 0$ or $j> N$. Here and throughout, $\ulcorner x\urcorner$ denotes the smallest integer greater than or equal to $x$. %and $M_{r,s}$ denote the $(r,s)$ entry of a matrix $M$.

Indeed, in convolutional neural networks (CNNs), a more general form of convolution is typically used where the input is a multichannel signal  represented by a 3-dimensional tensor $X=[x_{ijk}]\in \mathbb{R}^{N\times N\times g}$. Namely, the input $X$ has $g$ channels of $N\times N$ matrices. Then, a convolutional kernel is  represented by a 4-dimensional tensor $K=[k_{ijk\ell}]\in \mathbb{R}^{k\times k\times g  \times h}$ and the multichannel convolution of $K$ and $X$ produces  a 3 dimensional tensor output  $Y=[y_{ijk}]\in \mathbb{R}^{N \times N \times h}$,
%having $h$ channels. Thus, given a 4 dimensional tensor $K=[k_{ijk\ell}]\in \mathbb{R}^{k\times k\times g  \times h}$ as  a convolutional kernel, the multichannel convolution
as denoted  by $Y = K*X $  and defined by
\begin{equation}\label{eq:Mchannel}
    y_{r,s,c}= (K*X)_{r,s,c}=\sum_{d\in \{1,\cdots,g\}} \sum_{p\in \{1,\cdots,k\} }\sum_{q\in \{1,\cdots,k\}} x_{r-m+p,s-m+q,d}k_{p,q,d,c},
\end{equation}
where $m=\ulcorner k/2\urcorner$ and $x_{i,j,d}=0$ if $i\leq 0$ or $i> N$, or $j\leq 0$ or $j> N$. We will also call (\ref{eq:1channel}) a one-channel convolution, which is a special case of the multichannel convolution (\ref{eq:Mchannel}).

Clearly, the convolution operation is a linear transformation on $X$ and each convolutional kernel  corresponds to a linear transformation matrix. Indeed, the convolution equation $Y=K*X$ can be written as a matrix-vector product after reshaping $X$ and $Y$.  Let $vec(X)$   denote the reshape of $X$ into a vector as follows. If $X$ is a matrix,  $vec(X)$ is the column vector obtained  by  stacking the columns of  $X$ on top of one another. If $X$ is a tensor,  $vec(X)$ is the column vector obtained  by  stacking the columns of the flattening of $X$ along the first index (see \cite{golub2012} or Section \ref{sec:multi} for more details). Then, given a kernel $K$, there is a corresponding matrix $M$
%be the linear transformation matrix corresponding to the kernel $K$
such that
 \begin{equation*}
    vec(Y)=M \,vec(X).
\end{equation*}

In deep convolutional networks, a convolution kernel $K$ is a model parameter that defines the input and output $Y=K*X$ in a convolution layer and we need to construct $K$ by optimizing certain loss function ${\cal L} (K)$ with respect to $K$ (see Subsection \ref{app}
below). It is then desirable to construct $K$ such   that
\begin{equation}\label{eq:normxy}
\|vec(Y)\|_2\approx \|vec(X)\|_2
\end{equation}
%$\|vec(Y)\|_2\approx \|vec(X)\|_2$. % where $\|\cdot\|$ denotes a certain vector norm,
Such properties ensure bounded gradients which is critical in  training deep convolutional  neural networks \cite{hochreiter2001}.
Ideally then, we would like to construct $K$ so that the corresponding $M$ has orthonormal columns (i.e. $M^TM=I$), but this is  in general impossible because $M^TM=I$ involves $gN^2(gN^2+1)/2$ equations while $K$ has only $k^2 gh$ parameters with $k \ll N$ usually in neural networks. One known situation where an orthogonal $M$ can be constructed is the one-channel periodic convolution with full sized $N\times N$ kernel (i.e. k=N), for which the convolution becomes a diagonal matrix multiplication after discrete Fourier transforms; see \cite{sedghi2018} for example.
%Given a general kernel $K$ whose size is $k\times k\times g  \times h$ and the input whose size is $ N\times N\times g$, it is not known whether there exists an orthogonal $hN^2\times gN^2$ transformation matrix  corresponding to the $k\times k\times g  \times h$ kernel. It's theoretically difficult to answer this question.
For a multichannel convolution with a small kernel ($k\ll N$), a more realistic goal is to construct a kernel so that the singular values of the corresponding transformation matrix $M$ are bounded around $1$.

In this paper, we focus  on the numerical problem of constructing a convolutional kernel $K$ using a suitable regularization function to move the singular values of $M$ towards $1$.
One may consider explicitly  adding $max\{|\sigma_{max}(M)-1|,|\sigma_{min}(M)-1|\}$
%\begin{equation}\label{prob1}
%  max\{|\sigma_{max}(M)-1|,|\sigma_{min}(M)-1|\}
%\end{equation}
to ${\cal L} (K)$ as a penalty function,
where $\sigma_{max}(\cdot)$ and $\sigma_{min}(\cdot)$ denote the largest and respectively the smallest singular values of a matrix,  but with two objectives, this is more difficult to implement  however.
We   propose using ${\cal R}_\alpha (K) := \sigma_{max} (M^TM-\alpha I)$ (for some $\alpha>0$) as a penalty function for the regularization of the convolutional kernel.
We will show  that reducing  ${\cal R}_\alpha (K)$ keeps the largest singular value bounded from above and the smallest singular value from below. Equivalently up to a scaling, this reduces the condition number of $M$. We will then derive a gradient descent algorithm for minimizing ${\cal R}_\alpha (K)$. Numerical examples will be presented to illustrate our method.

%The goal is to minimize the following regularization term
%\begin{equation}\label{prob1}
%  max\{|\sigma_{max}(M)-1|,|\sigma_{min}(M)-1|\}
%\end{equation}
%where $M$ is the linear transformation matrix corresponding to kernel $K$.
%But the term \eqref{prob1} is hard to minimize directly. For a general matrix $A$, people  let the singular values of $A$ be bounded around $1$ through penalizing the term  $\sigma_{max} (A^TA-I)$, where $\sigma_{max}(\cdot)$ denote the largest singular value of a matrix, and $I$ is the identity matrix \cite{kova2008}. In this paper, we will use $\sigma_{max} (M^TM-\alpha I)$ as the penalty function, which include $\sigma_{max} (M^TM- I)$ as an example,  to let   the singular values of $M$ be bounded around $1$.

%???????Do we need this paragraph??????As we know, given a matrix, the singular values/eigenvalues are continuous functions depending on the entries of the matrix. We can calculate the partial derivatives of a singular value with respect to the entries, and  let  the singular values of a matrix  be bounded through changing the entries. But the transformation matrix $M$ corresponding to a convolutional kernel is structured, i.e., $M$ has special matrix structure. When changing the entries of $M$, we should preserve the special structure of $M$ such that the updated $M$ can still correspond to the same size convolutional kernel.  In this paper we will show how to preserve the special structure of $M$ when we minimize $\sigma_{max} (M^TM-\alpha I)$. The modification on $M$ is actually carried out on a special matrix manifold.

There have been many works devoted to enforcing  the orthogonality or spectral norm regularization on the weights of a neural network; see  \cite{brock2017,cisse,miyato2018,yoshida} and the references contained therein. %The difference between our paper and papers including \cite{brock2017,cisse,miyato2018,yoshida} and the references therein.
For a convolutional layer, some of these works enforce the constraint directly on the $h\times (gkk)$  matrix reshaped from the kernel $K\in \mathbb{R}^{k\times k\times g  \times h}$ without any clear impacts on $M$ \cite{brock2017,cisse}. \cite{miyato2018,yoshida} normalize a matrix reshaped from $K$ by its spectral norm.   \cite{sedghi2018} first constructs a full-sized kernel  under the periodic convolution that has a corresponding $M$ with bounded singular values and then   projects the full-sized convolutional kernel to a desirable small  one. This projection obviously may not preserve the desirable singular value bound of the original kernel.
%layer onto the set of layers obeying a bound on the operator norm of the layer and use numerical results to show this is an effective regularizer. A drawback of this method  is that projection can prevent the singular values of the transformation matrix being large but cannot avoid small singular values.
Compared with those approaches, our method works on the convolution kernel $K$ but regularize on the singular values of $M$. We also note that there are many works on constructing orthogonal weight matrices in the context of recurrent neural networks; see \cite{Arjo16,Helfrich2017,Madu19,Wisdom16} and the references contained therein, but we are concerned here with optimizing the singular values of a linear transformation defined by a convolution kernel rather than a general weight matrix.

The rest of the paper is organized as follows. In subsection \ref{app}, we will discuss the origin of our problem in deep learning.
%%As we have mentioned, the input channels and the output channels maybe more than one so the kernel is usually represented by a tensor $K\in \mathbb{R}^{k\times k\times g  \times h}$.
In Section~\ref{sec:one}, we first propose the penalty function and discuss its theoretical property. We then derive the gradient formula and propose the gradient descent algorithm for the one-channel case in Subsection~\ref{sec:one} and for the multichannel case in Subsection \ref{sec:multi}.
%  we first consider the case that the numbers of input channels and the output channels are both $1$.  In Section \ref{sec:multi}, we propose the penalty function and calculate the partial derivatives for the case of multichannel convolution.
In Section~\ref{sec:numer},  we present numerical results to show the effectiveness of the method. We end in Section~\ref{sec:conclu} with some concluding remarks.
% we will give some conclusions and point out some interesting work that could be done in future.

\subsection{Applications in deep learning}\label{app}

The regularization problem we consider arises in training of deep convolutional neural networks.
Convolutional neural network is one of the most widely used model of deep learning. A typical convolutional neural network consists of convolutional layers, pooling layers, and fully connected layers.
%In recent years, deep convolutional neural networks have been applied successfully in many fields, such as face recognition, self-driving cars, natural language understanding and speech recognition.
Training the neural networks is an optimization problem, which  seeks optimal weights (parameters) by reaching the minimum of loss function on the training data.
This can be described as follows: given a labeled data set $\{(X_i,Y_i)\}_{i=1}^{N}$, where $X_i$ is the input and $Y_i$ is the output, and a given parametric family of functions $\mathbb{F}=\{f(\Theta,X)\}$, where $\Theta$ denotes the parameters contained in the function, the goal of training the neural networks is to find the best parameters $\Theta$ such that $Y_i\approx f(\Theta,X_i)$ for $i=1,\cdots,N$. The practice is to minimize the so called loss function, e.g $\Sigma_{i=1}^{N}\|Y_i-f(\Theta,X_i)\|_2^2$ on the training data set.

For example, a typical convolutional neural network has $l$ convolutional layers parameterized by $l$ convolution kernels $K_p$ ($1\le p \le l$) and $m$ so-called fully-connected layers defined by weight matrices $W_q, (1\leq q\leq m)$; we omit the bias for the ease of notation. Then the output of the network can be written as $Y=f(K_1,K_2,\cdots, K_l,W_1,W_2,\cdots,W_m, X)$ and we train the network  by solving the following  optimization problem for the training dataset $\{(X_i,Y_i)\}_{i=1}^{N}$:
\begin{equation}\label{opt}
min_{K_1,K_2,\cdots, K_l,W_1,W_2,\cdots,W_m}  \frac{1}{N}\Sigma_{i=1}^{N}\|Y_i-f(K_1,K_2,\cdots, K_l,W_1,W_2,\cdots,W_m, X_i)\|.
\end{equation}
%\begin{equation*}
%\mbox{s.t.} \quad M_p^TM_p\approx I, \quad 1\leq p \leq l, \quad \mbox{and} \quad  W_q^TW_q\approx I,  \quad 1\leq q \leq m,
%\end{equation*}
%where  $I$ denotes the identity matrix.

Exploding and vanishing gradients are fundamental obstacles to solving (\ref{opt}) or training of deep neural networks \cite{hochreiter2001}. The singular values of the Jacobian of a layer bound the factor by which it changes the norm of the backpropagated signal. If these singular values are all close to $1$, then gradients neither explode nor vanish. This can also help improve the generalizability.
%to let  the singular values of the transformation matrix corresponding to a kernel are bounded around $1$.
Specifically, although the training of neural networks can be seen as an optimization problem, but the goal of training is not merely to minimize the loss function on training data set. In fact, the performance of the trained model on new data is the ultimate concern. That is to say, after we find the weights or parameters $\Theta$ through minimizing the loss function on training data set, we will use the weights $\Theta$ to get a neural network to predict the output or label for   new input data.
%Sometimes, the minimum on the training model is reached while the performance on test data is not satisfactory. A concept,  generalizability, is used to describe this phenomenon.
Generalizability, the ability of a network to extend its performance on the training data to new data, can  be improved through reducing the sensitivity of the output against the input data perturbation \cite{goodfellow2013,szegedy2014,tsuzuku2018,zhang,yoshida}. This again can be achieved through (\ref{eq:normxy}) and hence through regularizing the singular values of $M$.

\section{Regularization of Convolution Kernels}\label{sec:None}

One way to achieve \eqref{eq:normxy} is by minimizing $\sigma_{max} (M^TM-I)$ so that $M$ is close to being orthogonal. Since the  number of parameters in the convolution kernel $K$ may be relatively small, the minimum value with respect to  $K$ may not be  very close to $0$. Namely, enforcing $M$ to be nearly orthogonal may be too strong a condition to satisfy. Note that our goal to decrease $\sigma_{max}(M)$ while increasing $\sigma_{min}(M)$ is equivalent, up to a scaling, to decreasing the condition number of $M$. In light of this, we propose to minimize ${\cal R}_\alpha (K) := \sigma_{max} (M^TM- \alpha I)$ for some fixed $\alpha$. The following theorem justifies this approach.
%When this happens, how far are the singular values of $M$ from $1$? About this first remark, we give the following Lemma.

\begin{theorem}\label{remark1}
Let  $\alpha>0$ and $M \in \mathbb{R}^{m\times n}$ be such that ${\cal R}_\alpha (K) =\sigma_{max}(M^TM-\alpha I) < t \alpha$ for some $0< t \le 1$.  Then the largest and the smallest singular value of $M$, denoted by $\sigma_{max}(M)$ and $\sigma_{min}(M)$ respectively, satisfy that $$\sqrt{(1-t)\alpha} < \sigma_{min}(M)\leq \sigma_{max}(M) < \sqrt{(1+t)\alpha}.$$
In particular, $\kappa_2 (M) := \frac{\sigma_{max}(M)}{\sigma_{min}(M)} < \sqrt{\frac{1+t}{1-t}}.$
\end{theorem}
\begin{proof}
We use $\lambda_1(\cdot), \lambda_2(\cdot), \cdots, \lambda_m(\cdot)$ to denote all eigenvalues of an $m\times m$ matrix.
Since $M^TM-\alpha I$ is symmetric and $\sigma_{max}(M^TM-\alpha I)<t \alpha$, then for all $i =1,2, \cdots, m$,  we have
\begin{equation*}
    - t\alpha<\lambda_i(M^TM-\alpha I)<t\alpha,
\end{equation*}
and thus
\begin{equation*}
   (1-t)\alpha<\lambda_i(M^TM)< (1+t)\alpha.
\end{equation*}
Therefore we have
\begin{equation*}
    \sqrt{(1-t)\alpha}<\sigma_{min}(M)\leq \sigma_{max}(M)<\sqrt{(1+t)\alpha}.
\end{equation*}
The bound on the condition number $\kappa_2 (M)$ follows immediately.
\end{proof}

Theorem~\ref{remark1} suggests that reducing ${\cal R}_\alpha (K)$ to a value less than $\alpha$ is sufficient to keep  $\sigma_{max}(M)$ bounded above and $\sigma_{min}(M)$ bounded below. Then, the  reduction in ${\cal R}_\alpha (K)$ needed to maintain boundedness of the singular values may be less by using a larger value of $\alpha$.
We next discuss the gradient descent algorithm to minimize ${\cal R}_\alpha (K)$. We first discuss the one-channel convolution and then present the generalization to multichannel cases.
%tells us that  when carrying out the gradient descent algorithm for  the penalty function $\sigma_{max} (M^TM-\alpha I)$, even if we can't have  $\sigma_{max}(M^TM-\alpha I)\approx 0$,  we can still ensure that the singular values of $M$ are in an bounded interval as long as $\sigma_{max} (M^TM-\alpha I)$ is bounded.

\subsection{One-channel convolution}\label{sec:one}

We first consider the one-channel convolution (\ref{eq:1channel}), i.e. in the context of (\ref{eq:Mchannel}), the numbers of input channels and the output channels are both $1$. In this case, the  kernel is a $k\times k$ matrix and the input and the output are $N\times N$ matrices. For the ease of notation, we use a $3\times 3$ convolution kernel to illustrate the associated transformation matrix. Let the convolution kernel $K$ be
\begin{eqnarray*}
K=\left(\begin{array}{ccc}
k_{11} & k_{12} & k_{13} \\
k_{21} & k_{22} & k_{23} \\
k_{31} & k_{32} & k_{33}
\end{array}\right).
\end{eqnarray*}
Then the transformation matrix corresponding to the convolution operation is
\begin{eqnarray}\label{conv0}
M=M(K):=\left(
  \begin{array}{cccccc}
    A_0 & A_{-1} & 0 & 0 & \cdots & 0 \\
    A_1 & A_0 & A_{-1} & \ddots & \ddots & \vdots \\
    0 & A_1 & A_0 & \ddots & \ddots & 0 \\
    0 & \ddots & \ddots & \ddots & A_{-1} & 0 \\
    \vdots & \ddots & \ddots & A_1 & A_0 & A_{-1} \\
    0 & \cdots & 0 & 0 & A_1 & A_0 \\
  \end{array}
\right)
\end{eqnarray}
where
\begin{eqnarray*}
A_0=\left(
  \begin{array}{cccccc}
    k_{22} & k_{23} & 0 & 0 & \cdots & 0 \\
    k_{21} & k_{22} & k_{23} & \ddots & \ddots & \vdots \\
    0 & k_{21} & k_{22} & \ddots & \ddots & 0 \\
    0 & \ddots & \ddots & \ddots & k_{23} & 0 \\
    \vdots & \ddots & \ddots & k_{21} & k_{22} & k_{23} \\
    0 & \cdots & 0 & 0 & k_{21} & k_{22}
  \end{array}
\right),\quad
A_{-1}=\left(
  \begin{array}{cccccc}
    k_{32} & k_{33} & 0 & 0 & \cdots & 0 \\
    k_{31} & k_{32} & k_{33} & \ddots & \ddots & \vdots \\
    0 & k_{31} & k_{32} & \ddots & \ddots & 0 \\
    0 & \ddots & \ddots & \ddots & k_{33} & 0 \\
    \vdots & \ddots & \ddots & k_{31} & k_{32} & k_{33} \\
    0 & \cdots & 0 & 0 & k_{31} & k_{32}
  \end{array}
\right),
\end{eqnarray*}
\begin{eqnarray*}
A_{1}=\left(
  \begin{array}{cccccc}
    k_{12} & k_{13} & 0 & 0 & \cdots & 0 \\
    k_{11} & k_{12} & k_{13} & \ddots & \ddots & \vdots \\
    0 & k_{11} & k_{12} & \ddots & \ddots & 0 \\
    0 & \ddots & \ddots & \ddots & k_{13} & 0 \\
    \vdots & \ddots & \ddots & k_{11} & k_{12} & k_{13} \\
    0 & \cdots & 0 & 0 & k_{11} & k_{12}
  \end{array}
\right).
\end{eqnarray*}
In particular, $M$  is a $N^2\times N^2$ doubly  block banded Toeplitz matrix, i.e., a block banded Toeplitz matrix with its blocks are banded Toeplitz matrices \cite{jin2002}.
%Let $n=N^2$ and let $\mathcal{T}$  denote the set of all matrices like $M$ in \eqref{conv0}, i.e., doubly  block banded Toeplitz matrices with the fixed bandwidth.

To minimize ${\cal R}_\alpha (K) =\sigma_{max}(M^TM-\alpha I)$, we derive a formula for its gradient with respect to $K$, i.e., $\partial \sigma_{max} (M^TM-I)/\partial k_{p,q}$ with $M=M(K)$ being the transformation matrix defined from $K$ in (\ref{conv0}) for each entry $k_{p,q}$ of the convolution kernel.  Our result provides a framework to use ${\cal R}_\alpha (K)$ as a regularization term in the optimization of ${\cal L} (K)$ in convolutional neural networks.
%calculate the gradient of the penalty function about transformation matrix versus the convolution kernel. People can construct other penalty function about $M$ and get the gradient descent method when training their convolutional networks.
To compute the gradient, we need the following classical result on the first order perturbation expansion about a simple singular value; see \cite{stewart}.
\begin{lemma}\label{lem1}
Let $\sigma$ be a simple singular value of $A = [a_{ij}] \in\mathbb{R}^{m \times m}$
 ($n\geq p$) with normalized left and right singular vectors $u$ and $v$.  Then $\partial \sigma/\partial a_{ij}$  is $u(i)v(j)$, where $u(i)$ is the $i$-th entry of vector $u$ and $v(j)$ is the $j$-th entry of vector $v$.
\end{lemma}

For our situation, we need to consider perturbation of $M^T M$ when $M$ is changed. Clearly, if an entry $m_{ij}$  changes, only the entries belonging to $j$-th row or $j$-th volume of the matrix $M^TM$ are affected.
Actually, we have the following lemma.

\begin{lemma}\label{lem2}
Let $M=[m_{ij}]\in \mathbb{R}^{m\times n}$ and let $\sigma_{max}(M^TM-\alpha I)$ be the largest singular value of $M^TM-\alpha I$ with $u$ and $v$  normalized left and right singular vectors. Assuming $\sigma_{max}(M^TM-\alpha I)$ is simple and positive, we have
\begin{equation}\label{dm}
\frac{\partial \sigma_{max}(M^TM-\alpha I)}{\partial m_{ij}} =
v(j)u^T M^Te_i +u(j)e_i^T M v
\end{equation}
%
%For  $A = [a_{ij}] \in\mathbb{R}^{m \times n}$, we have
%\begin{equation}\label{eq:ata}
%\frac{\partial(A^TA)}{\partial a_{ij}}=A^T(e_ie_j^T)+(e_je_i^T)A
%=\left( \begin{array}{cccccccc}
%   0 & \cdots & \cdots & 0 & a_{i1} & 0 & \cdots & 0 \\
%   \vdots & \vdots & \vdots & \vdots & a_{i2} & \vdots & \vdots& \vdots \\
%  \vdots & \vdots & \vdots& \vdots & \vdots & \vdots & \vdots & \vdots \\
%   0 & \cdots &\cdots & 0 & a_{i,j-1} & 0 & \cdots & 0 \\
%   a_{i1} & a_{i2} & \cdots & a_{i,j-1} &2 a_{ij} & a_{i,j+1} & \cdots &  a_{in} \\
%   0 & \cdots & \cdots & 0 & a_{i,j+1} & 0 & \cdots &0 \\
%   \vdots & \vdots & \vdots & \vdots & \vdots & \vdots & \vdots & \vdots \\
%   0 & \cdots & \cdots & 0 & a_{in} & 0 & \cdots & 0
% \end{array}
% \right).
%\end{equation}
where $e_k$ denotes the $k$-th column of the $n\times n$ identity matrix.
%$\partial(A^TA)_{s,t}/ \partial a_{ij}$ is the $(s,t)$ entry of the  matrix $D=A^T(e_ie_j^T)+(e_je_i^T)A$, where
\end{lemma}
%\begin{proof}
%Let $A=[a_{ij}]=M^TM-\alpha I$. A direct calculation yields
%\[
%\frac{\partial(A)}{\partial a_{ij}} =
%A^T \frac{\partial(A)}{\partial a_{ij}} + \frac{\partial(A^T)}{\partial a_{ij}} A  =A^T(e_ie_j^T)+(e_je_i^T)A.
%\]
%\end{proof}
\begin{proof}
Let $A=[a_{ij}]=M^TM-\alpha I$. A direct calculation yields
\[
\frac{\partial A}{\partial m_{ij}} =
M^T \frac{\partial M}{\partial m_{ij}} + \frac{\partial(M^T)}{\partial m_{ij}} M  =M^T(e_ie_j^T)+(e_je_i^T)M.
\]
It follows from  this, lemma~\ref{lem1}  and the chain rule that
\begin{eqnarray*}\label{derivative1}
 \nonumber  \frac{\partial \sigma_{max}(A)}{\partial m_{ij}}  &=& \sum_{s=1}^n \sum_{t=1}^{n} \frac{\partial \sigma_{max}(A)}{\partial a_{s,t}}\frac{\partial a_{s,t}}{\partial m_{ij}} \\
 &=& \sum_{s=1}^n \sum_{t=1}^{n} u(s) v(t) \frac{\partial a_{s,t}}{\partial m_{ij}}
  \\
 &=& u^T \frac{\partial A}{\partial m_{ij}} v \\
 &=& u^T (M^T(e_ie_j^T)+(e_je_i^T)M) v \\
 &=& v(j)u^T M^Te_i +u(j)e_i^T M v %\\
%  &=& \sum_{t=1}^{n} u_j v(t) m_{it}+\sum_{s=1,\cdots,n}u(s) v(j) m_{is}.
\end{eqnarray*}
\end{proof}

We can now derive a formula for the gradient descent of $\sigma_{max}(M^TM-I)$ with respect to the convolution kernel $K$ as follows.

\begin{theorem}\label{theo}
Assume the largest singular value of $M^TM-\alpha I$, denoted by $\sigma_{max}(M^TM-\alpha I)$, is simple and positive, where $M=[m_{ij}]=M(K)\in \mathbb{R}^{n\times n}$ is the doubly  block banded Toeplitz matrix \eqref{conv0} corresponding to a one channel convolution kernel $K=[k_{ij}]\in\mathbb{R}^{k\times k}$. Assume $u$ and $v$  are normalized left and right singular vectors of $M^TM-\alpha I$ associated with $\sigma_{max}(M^TM- \alpha I)$. Given $(p,q)$, if $\Omega_{p,q}$ denotes the set of all indexes $(i,j)$ such that $m_{ij}=k_{p,q}$,  we have
\begin{equation}\label{derivative3}
    \frac{\partial \sigma_{max}(M^TM- \alpha I)}{\partial k_{p,q}}= \sum_{(i,j)\in\Omega_{p,q}}(\sum_{t=1}^{n}u(j)v(t)m_{it}+\sum_{s=1}^{n}u(s)v(j)m_{is}).
\end{equation}
\end{theorem}
\begin{proof}
%Let $A=[a_{ij}]=M^TM-\alpha I$. It follows from  lemma~\ref{lem1}, lemma~\ref{lem2} and the chain rule that
%\begin{eqnarray*}\label{derivative1}
% \nonumber  \frac{\partial \sigma_{max}(A)}{\partial m_{ij}}  &=& \sum_{s=1}^n \sum_{t=1}^{n} \frac{\partial \sigma_{max}(A)}{\partial a_{s,t}}\frac{\partial a_{s,t}}{\partial m_{ij}} \\
% &=& \sum_{s=1}^n \sum_{t=1}^{n} u(s) v(t) \frac{\partial a_{s,t}}{\partial m_{ij}}
%  \\
% &=& u^T \frac{\partial A}{\partial m_{ij}} v \\
% &=& u^T \frac{\partial M^TM}{\partial m_{ij}} v \\
% &=& u^T (M^T(e_ie_j^T)+(e_je_i^T)M) v \\
% &=& u^T M^Te_i v_j+u_je_i^T M v \\
%  &=& \sum_{t=1}^{n} u_j v(t) m_{it}+\sum_{s=1,\cdots,n}u(s) v(j) m_{is}.
%\end{eqnarray*}
%For a matrix $M\in\mathcal{T}$,  The value of $k_{p,q}$ will appear in different  $(i,j)$ indexes. We use $\Omega$ to denote this index set, i.e., for each $(i,j)\in\Omega$ , we have   $m_{ij}=k_{p,q}$.
%
From (\ref{conv0}),  $m_{ij}$ is either 0 or equal to some $k_{p,q}$. Indeed, $m_{ij}=k_{p,q}$ if and only if $(i,j)\in\Omega_{p,q}$. Now, applying the chain rule  to calculate $\partial \sigma_{max}(M^TM-I)/\partial k_{p,q}$ and using Lemma \ref{lem2}, we have
\begin{eqnarray*}\label{derivative2}
% \nonumber to remove numbering (before each equation)
  \nonumber  \frac{\partial \sigma_{max}(M^TM-I)}{\partial k_{p,q}}
  &=&  \sum_{i=1}^n \sum_{j=1}^{n}  \frac{\partial \sigma_{max}(M^TM-I)}{\partial m_{ij}} \frac{\partial m_{ij}}{\partial k_{p,q}} \\
  &=& \sum_{(i,j)\in\Omega_{p,q}} \frac{\sigma_{max}(M^TM-I)}{\partial m_{ij}}\\
  &=& \sum_{(i,j)\in\Omega_{p,q}}(\sum_{t=1}^{n}u(j)v(t)m_{it}+\sum_{s=1}^{n}u(s)v(j)m_{is}).
\end{eqnarray*}
\end{proof}

We remark that $M^TM-I$ in the above theorem is a symmetric matrix. Then its largest singular value $\sigma_{max}(M^TM-\alpha I)$ is either its largest eigenvalue or the absolute value of its smallest eigenvalue. Then the left singular vector $u$ is equal to $v$ or $-v$ respectively.

With the gradient, we can minimize $\sigma_{max}(M^TM-\alpha I)$ with respect to $K$ using an optimization method. In convolutional neural networks, the number of parameters are usually so large that a first order method such as gradient descent is typically used. We therefore also consider the gradient descent method
\[
K \leftarrow K - \lambda \nabla \sigma_{max}(M^TM-\alpha I)
\]
where $\lambda$ is a step size parameter called learning rate. Then, at each step of iteration, to get the gradient, we need to compute $\sigma_{max}(M^TM-\alpha I)$ and the associated left and right singular vector. Although the dimension of $M$ is large, $M$ is quite sparse and we can compute a few largest singular values efficiently with Krylov subspace methods \cite{BaglamaReichel:svd,arpack,LiangYe:svdifp}. Moreover, a Toeplitz matrix can be embedded into a circulant matrix and the matrix-vector multiplication can be efficiently computed using the fast Fourier transform  by exploiting the convolution structure;  see \cite{chan1996,jin2002} and the reference therein.
Nevertheless, this may be computationally costly, since the gradient descent algorithm may require a large number of iterations and hence repeated computations of the gradients. On the other hand, with $\lambda$ usually being very small,  each step of iterations involve a small change in $K$ and  $\sigma_{max}(M^TM-\alpha I)$. We therefore suggest to use the power method to update the largest singular value and singular vectors at each step of gradient descent step. Our experiences indicates that a few iterations are usually sufficient. With this approach, one potential issue is that the largest singular value may be overtaken by the second largest singular value during the iterations. As a remedy, we keep and update the two largest singular values and singular vectors and select the larger one after each update as the largest singular value.
We will present a detailed algorithm in Section \ref{sec:multi} in the more general context of multichannel convolution.

\subsection{Multi-channel convolution}\label{sec:multi}
We now generalize the result in Section \ref{sec:one} to  multichannel convolutions. Consider a 4-dimensional tensor convolution kernel $K=[k_{i,j,k,l}]\in \mathbb{R}^{k\times k\times g  \times h}$ and a 3-dimensional tensor  input $X=[x_{i,j,k}] \in \mathbb{R}^{N\times N\times g}
$.  Let $Y=[y_{i,j,k}] \in \mathbb{R}^{N \times N \times h}$ be the 3-dimensional tensor output produced by the convolution $Y=K*X$ as defined in (\ref{eq:Mchannel}). Let $vec(X)$ denote the vectorization of $X$, i.e.
\[
vec(X) = [x_{:,1,1}^T,\ldots,x_{:,N,1}^T,x_{:,1,2}^T,\ldots,x_{:,N,2}^T, \ldots,x_{:,1,g}^T,\ldots,x_{:,N,g}^T]^T
\]
where we have used MATLAB notation $x_{:, i, j}:=[x_{1,i,j},\ldots,x_{N,i,j}]^T$.

In  this notation, the convolution operation is expressed as  $vec(Y)=\M vec(X)$, where
\begin{eqnarray}\label{conv2}
\M=\M(K):=\left(\begin{array}{cccc}
M_{(1)(1)} & M_{(1)(2)} & \cdots& M_{(1)(g)} \\
M_{(2)(1)} & M_{(2)(2)} & \cdots & M_{(2)(g)} \\
\vdots &\vdots &\cdots  &\vdots \\
M_{(h)(1)}& M_{(h)(2)}& \cdots &M_{(h)(g)}
\end{array}\right),
\end{eqnarray}
and   %$B_{(c)(d)}\in \mathcal{T}$, i.e.,
$M_{(c)(d)} = M(K_{:,:,d,c})$ is a $N^2\times N^2$ doubly  block banded Toeplitz matrix as defined in (\ref{conv0}) from the 2-dimensional kernel  $K_{:,:,d,c}$. Namely, $M_{(c)(d)}$  is the transformation matrix corresponding to 2-dimensional kernel $K_{:,:,d,c}$ that convolutes with the  $d$-th input channel to produce the $c$-th output channel.
%$n=N^2$ and use $\mathcal{BT}$ to denote the set of all matrices like $M$ in \eqref{conv2}, i.e., matrices whose blocks are doubly  block banded Toeplitz matrices with the fixed bandth.

As in Section \ref{sec:one}, we are interested in minimizing $\sigma_{max} (\M^T\M-\alpha I)$ with respect to $K$. We can easily generalize Theorem \ref{theo} to the multichannel case  as follows; the proof follows from Lemma \ref{lem2} as in that of  Theorem \ref{theo} and is omitted here.
\begin{theorem}\label{theo2}
Assume the largest singular value of $\M^T\M-\alpha I$  is simple and positive, where $\M=\M(K)$   is the structured matrix corresponding to the multichannel convolution kernel $K\in\mathbb{R}^{k\times k\times g \times h}$ as defined in (\ref{conv2}).
%each block $B_{(y)(z)}$ is a $n\times n$ doubly  block banded Toeplitz matrix corresponding to the portion $K_{:,:,z,y}$.
Assume $u, v$ are the normalized left  and right singular vectors corresponding to $\sigma_{max}(\M^T\M-I)$. Given $(p,q,z,y)$, if $\Omega_{p,q,z,y}$ is the set of all indexes $(i,j)$ such that $m_{ij}=k_{p,q,z,y}$,  we have
\begin{equation}\label{derivative4}
    \frac{\partial \sigma_{max}(M^TM-I)}{\partial k_{p,q,z,y}}= \sum_{(i,j)\in\Omega_{p,q,z,y}}(\sum_{t=1}^{g*N^2}u(j)v(t)m_{it}+\sum_{s=1}^{g*N^2}u(s)v(j)m_{is}).
\end{equation}
where $m_{ij}$ is the $(i,j)$ entry of $\M$.
\end{theorem}

As discussed at the end of Subsection \ref{sec:one}, we can use the derivative in a gradient descent iteration to minimize ${\cal R}_\alpha (K) $  with respect to $K$. We  give a detailed description of the full procedure in the following algorithm.

\begin{algorithm} \label{alg1}
\noindent \textbf{Gradient Descent for ${\cal R}_\alpha (K) =\sigma_{max} (M^TM-\alpha I)$.}
\begin{tabbing}
aaaaa \= bbbb\= \kill
1. \> Input: an initial kernel $K\in \mathbb{R}^{k\times k\times g  \times h}$, input size $N\times N\times g$   and learning rate $\lambda$.\\
%2. \>Form the transformation matrix $M$ as \eqref{conv2}, get the index sets $\Omega_{p,q,z,y}$ corresponding to $k_{p,q,z,y}$, i.e., $k_{p,q,z,y}=m_{i,j}$ holds for all $(i,j)\in \Omega_{p,q,z,y}$.\\
2. \>Compute $(\sigma_1, u_1, v_1)$ and $(\sigma_2, u_2, v_2)$, i.e. the first and the second largest singular values and the associated \\
 \> normalized left and right singular vectors of $\M^T\M-\alpha I$ where $\M=\M(K)$ is defined in (\ref{conv2});\\
3. \> set $u=u_1, v=v_1$.\\
4. \>While not converged:\\
4. \>\>Compute $G= [\frac{\partial \sigma_{max}(M^TM-I)}{\partial k_{p,q,z,y}}]_{p,q,z,y=1}^{k,k,g,h} $, by \eqref{derivative4}; \\
5. \>\> Update $K=K-\lambda G$;\\
6. \>\>Update  $(\sigma_1, u_1, v_1)$ and $(\sigma_2, u_2, v_2)$ using the power method;  \\
7. \>\>If $\sigma_1 \ge \sigma_2$,  $u=u_1, v=v_1$; \\
   \>\>else, $u=u_2, v=v_2$; \\
8. \>End
\end{tabbing}
\end{algorithm}

\section{Numerical experiments}\label{sec:numer}
In this section, we present two numerical examples to illustrate our method. We study performance of our method with respect to different sizes of convolution kernels and different values of $\alpha$ in the regularization function ${\cal R}_\alpha (K)$.  All  numerical  tests were performed on a PC with MATLAB R2016b.

In both examples, we start from a random kernel with each entry uniformly distributed on $[0, 1]$, i.e. in MATLAB, ${\tt K= rand(k,k,g,h)}$ with ${\tt rand('state',1)}$. We then minimize ${\cal R}_\alpha (K)$ using Algorithm \ref{alg1} and we demonstrate the beneficial effect of reducing the condition number of $M$, or  decreasing $\sigma_{max}(M)$ while maintaining $\sigma_{min}(M)$. In our numerical experiments, we have used  $\lambda=0.01$. At step 6 of Algorithm \ref{alg1} to update the singular values of $M$, we have experimented using the power method with two iterations  as well as using the full SVD decomposition. We have found the results are comparable and we present the one based on the power method only.

{\sc Example} 1: We consider kernels of different sizes with $3\times 3$ filters in this example, namely $K\in \mathbb{R}^{3\times 3\times g  \times h}$ for various values of $g,h$.  For each kernel, we use the input data matrix of size   $15\times 15\times g$.  We use the penalty function  ${\cal R}_1 (K) = \sigma_{max}(M^TM- I)$. We present in Figure~\ref{fig1} the results of $3\times 3\times 3  \times 1$, $3\times 3\times 1  \times 3$, $3\times 3\times 3  \times 6$, and $3\times 3\times 6  \times 3$ kernels. In the figures, we have shown the convergence of $\sigma_{max}(M^TM- I)$ (red solid line) on the right axis scale, and  $\sigma_{max}(M)$ (blue solid line), $\sigma_{min}(M)$ (blue dashed line), and the condition number $ \kappa(M)$ (blue dotted line) on the left axis scale.

For all kernel sizes, $\sigma_{max}(M^TM- I)$ converges well within 20 iterations. The condition number $ \kappa(M)$  and $\sigma_{max}(M^TM- I)$ decreases accordingly. $\sigma_{min}(M)$ does not change significantly, however. It appears minimizing ${\cal R}_1 (K)$ is more effective in decreasing  $ \kappa(M)$  and $\sigma_{max}(M^TM- I)$ but less so in increasing $\sigma_{min}(M)$. The kernel sizes mainly affect the final converged values but not the convergence behavior.

%and $\sigma_{max}(M^TM- I)$ on the right side for each kernel to see the change of these values as the iteration steps become larger.  At the initial step $\sigma_{min}(M)$ is close to 1 and as the iteration steps become larger, we see $\sigma_{min}(M)$ almost  unchanged but $\sigma_{max}(M)$ decreases and becomes closer to $1$. Gradually  $\sigma_{max}(M)$ and $\sigma_{min}(M)$ are all close to 1.
\begin{figure}[h]
  \centering
  \includegraphics[width=1.00\textwidth]{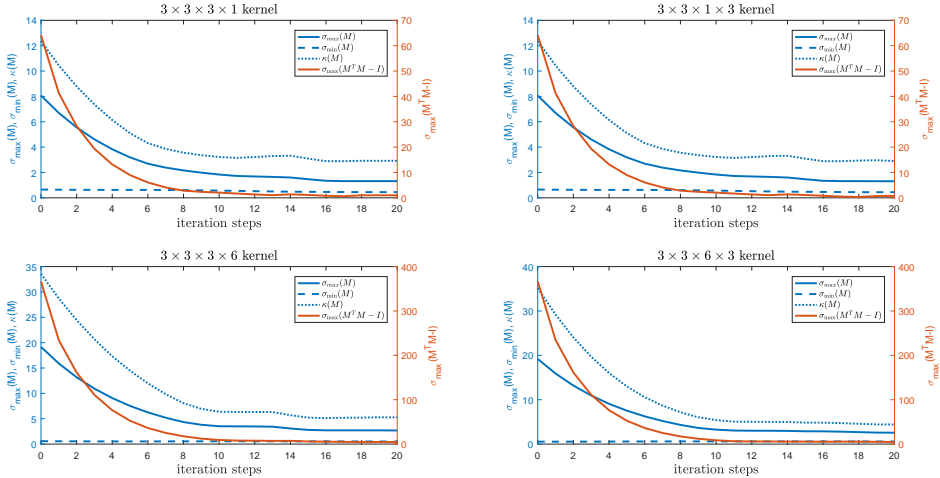}
\caption{\text{\small{Example 1: Convergence of $\sigma_{max}(M), \sigma_{min}(M), \kappa(M), \sigma_{max}(M^TM- I)$  for four kernel sizes}}}
\label{fig1}
  \end{figure}

{\sc Example} 2: We consider kernels of size $3\times 3\times 3  \times 1$ and use ${\cal R}_\alpha (K)= \sigma_{max}(M^TM-\alpha I)$ with $\alpha =0.1, 1, 5,$ and $ 10$.   We present in Figure~\ref{fig2} the convergence of $\sigma_{max}(M^TM- I)$ (red solid line) on the right axis scale, and  $\sigma_{max}(M)$ (blue solid line), $\sigma_{min}(M)$ (blue dashed line), and the condition number $ \kappa(M)$ (blue dotted line) on the left axis scale.

For all values of $\alpha$, $\sigma_{max}(M^TM- \alpha I)$ converges to a value dependent on $\alpha$. The condition number $ \kappa(M)$  and $\sigma_{max}(M^TM- \alpha I)$ decreases accordingly. For the larger values of $\alpha$,  the convergence appears faster. For example, for $\alpha=5 $ and $10$, $\sigma_{max}(M^TM- \alpha I)$ reaches minimum a little below the values of $\alpha$ at the $6$th and the $4$th iteration. Even though the minimum values are also larger than other cases, it has similar effect in reducing $ \kappa(M)$  and $\sigma_{max}(M^TM- \alpha I)$ as suggested by Theorem \ref{remark1}. An interesting observation is that after $\sigma_{max}(M^TM- \alpha I)$ reaches a value smaller than $\alpha$, it increases back to a level of $\alpha$. It appears there may be a theoretical barrier to reducing $\sigma_{max}(M^TM- \alpha I)$ much below $\alpha$.

\begin{figure}[h]
  \centering
  \includegraphics[width=1.00\textwidth]{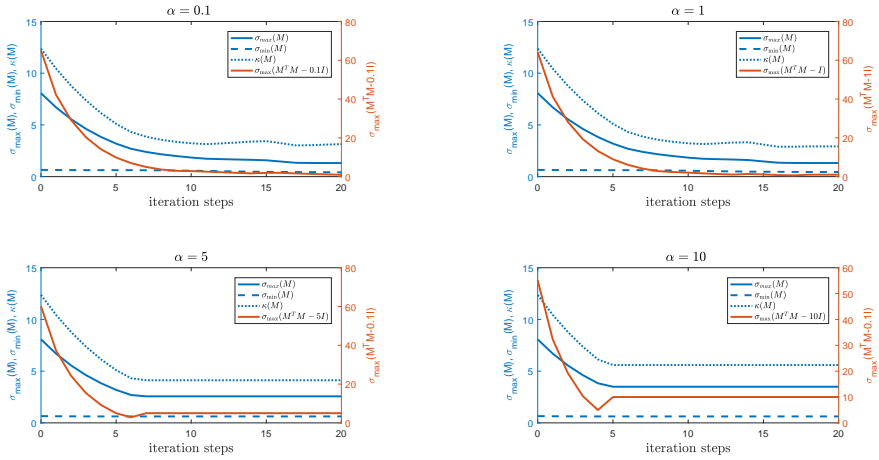}
\caption{\text{\small{Example 2: Convergence of $\sigma_{max}(M), \sigma_{min}(M), \kappa(M), \sigma_{max}(M^TM- \alpha I)$ for different $\alpha$}}}
\label{fig2}
  \end{figure}

\section{Conclusions}\label{sec:conclu}
In this paper, we have considered how to regularize the weights of convolutional layers in convolutional neural networks.
The goal  is to constrain  the singular values of the structured transformation matrix corresponding to a convolutional kernel to be neither too large nor too small. We have devised the penalty function and proposed the gradient decent method for the convolutional kernel to achieve this. Numerical examples demonstrate its effectiveness for different size of convolution kernels. We have also proposed a more general penalty function  ${\cal R}_\alpha (K)$ and have observed some interesting behavior with respect to the choice of $\alpha$. It will be interesting to further investigate this, which is left to a future work.

\section{Acknowledgements}
The authors are grateful to Professor Xinguo Liu at Ocean University of China and Professor Beatrice Meini at University of Pisa for their valuable suggestions.

\end{document}